\documentclass[letterpaper]{article}
\usepackage{aaai}
\usepackage{times}
\usepackage{helvet}
\usepackage{courier}
\usepackage{amsthm}
\usepackage{algorithm}
\usepackage{algorithmic}
\usepackage{multirow}
\usepackage{graphics}
\usepackage{amsmath}

\newtheorem{definition}{Definition}
\newtheorem{lemma}{Lemma}
\newtheorem{theorem}{Theorem}
\newtheorem{corollary}{Corollary}

\DeclareMathOperator*{\argmax}{\arg\!\max}

\begin{document}
%
\title{Learning of Agent Capability Models with Applications \\in Multi-agent Planning}
\author{Yu Zhang and Subbarao Kambhampati\\
School of Computing and Informatics\\
Arizona State University\\
Tempe, Arizona 85281 USA\\
\{yzhan442,rao\}@asu.edu\\
}
\nocopyright
\maketitle

\begin{abstract}
\begin{quote}

One important challenge for a set of agents to achieve more efficient collaboration is for these agents to maintain proper models of each other. 
An important aspect of these models of other agents is that they are often partial and incomplete.
Thus far, there are two common representations of agent models: MDP based and action based,
which are both based on action modeling. 
In many applications, agent models may not have been given, and hence must be learnt. 
While it may seem convenient to use either MDP based or action based models for learning, 
in this paper, we introduce a new representation based on capability models, which has several unique advantages. 
First, we show that learning capability models can be performed efficiently online via Bayesian learning, 
and the learning process is robust to high degrees of incompleteness in plan execution traces (e.g., with only start and end states). 
While high degrees of incompleteness in plan execution traces presents learning challenges for MDP based and action based models, 
capability models can still learn to {\em abstract} useful information out of these traces.
As a result, capability models are useful in applications in which such incompleteness is common, 
e.g., robot learning human model from observations and interactions.
Furthermore, when used in multi-agent planning (with each agent modeled separately), 
capability models provide flexible abstraction of actions.
The limitation, however, is that the synthesized plan is incomplete and abstract. 

\end{quote}
\end{abstract}

\section{Introduction}
\label{sec:intro}

One important challenge for a set of agents to achieve more efficient collaboration is for these agents to maintain proper models of others. 
These models can be used to reduce communication and collaboration efforts.
In many applications, agent models may not have been given, and hence must be learnt. 
Thus far, there are two common representations of agent models: 
MDP based \cite{Puterman:1994} and action based \cite{Fox03}, which are both based on {\em action modeling}.
In this paper, we introduce a new representation based on capability models.
We represent a capability as the ability to achieve a partial state given another partial state.
Contrasting to action modeling in MDP based and action based models, 
a capability implicitly corresponds to a set of action sequences (or plans) that satisfy a specification, 
described as the transition between two partial states. 
Each such action sequence is called an {\em operation}. 
The capability model represents a model that captures the probabilities of the existence of an operation for these specifications. 

Compared to MDP based and action based models, 
which do not naturally represent incomplete and abstract models,
capability model has its unique benefits and limitations. 
In this aspect, capability models should not be considered as competitors to these more complete models.
Instead, they are more useful when only incomplete and abstract information can be obtained for the models, e.g., human models. 

The representation of the capability model is a generalization of a two time slice dynamic Bayesian network ($2$-TBN).
Each node at the first level (also called a {\em fact node}) represents a variable that is used in the specification of a world state. 
The underlying structure, i.e., the (potentially partial) causal relationships between the fact nodes, 
is assumed to be provided by the domain expert. 
Furthermore, for each fact node, there is also a corresponding node that represents the final value of this fact as the result of an operation.
In this paper, these corresponding nodes are called {\em eventual} nodes or e-nodes. 
The links between the e-nodes, as well as the links from the fact nodes to the e-nodes in the capability model, 
follow the same causal relationships (i.e., links) between the fact nodes.
Furthermore, each fact node is also connected to its corresponding e-node (from the fact node to e-node). 
Figure \ref{fig:cap-model} presents a capability model. 
One remark is that learning causal relationships to dynamically update the structures of capability models is possible \cite{White199099},
which is to be considered in future work.
We argue that this information is easier to obtain for certain models, e.g., human models.
Instead of modeling individual human actions, it is easier to associate causes with effects.
For example, given a cup of water, the effect could be a cup of milk or coffee, 
even though complex sequences of actions may be required for these effects.

\begin{figure}
\centering
{
    \includegraphics{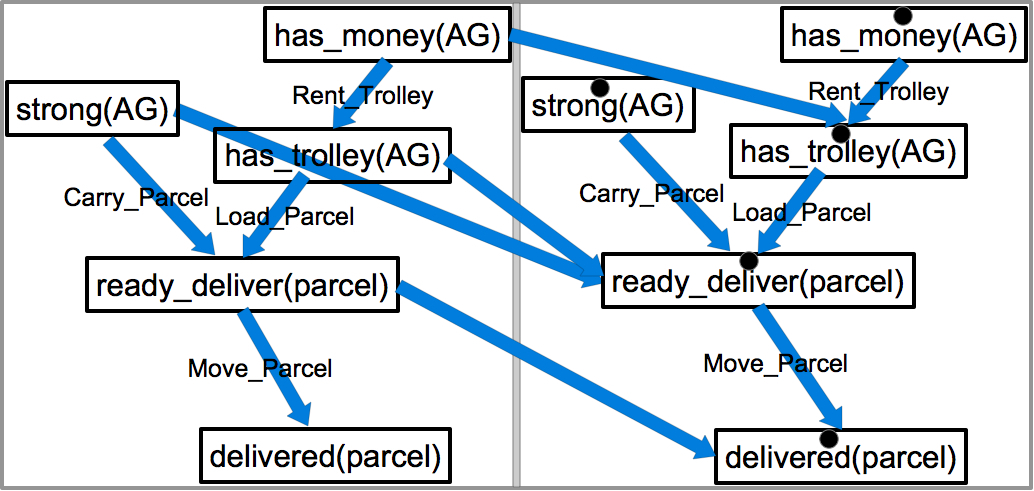}
}
\caption{Capability model (as a generalized $2$-TBN) for a human delivery agent (denoted as $AG$) in our motivating example with $5$ fact node and e-node pairs. 
Each fact node corresponds to a variable in the specification of the agent state and related world state. 
e-nodes are labeled with a dot on top of the variable.
Links are labeled with action names where applicable. 
Links from the fact nodes to the corresponding e-nodes are not shown for a cleaner representation.}
\label{fig:cap-model}
\end{figure}

Learning the parameters of capability models can be performed efficiently online via Bayesian learning. 
One of the unique advantages of capability model is that its learning process is robust to high degrees of incompleteness in plan execution traces. 
This incompleteness occurs commonly.
For example, the execution observer may not always stay close enough to the executing agent,
and the executing agent may not report the full traces \cite{Zhuo:2013}.
Capability models can even learn when only the initial and final states are given for plan executions.
While high degrees of incompleteness in plan execution traces presents learning challenges for MDP based and action based models, 
capability models can still learn to {\em abstract} useful information out of these traces, 
e.g., the probability of the existence of a plan (operation) to achieve a final (partial) state given an initial (partial) state.
When more information about the traces is given, it can be used to refine the capability models as with the other models.
Hence, capability model is useful for applications in which the agent models must be learnt from incomplete traces, 
e.g., robot learning human model from observations and interactions. 

Furthermore, compared to MDP based and action based representations, 
capability models provide more flexible abstraction of actions. 
As a result, when using capability models in multi-agent planning, 
the planning performance can be potentially improved. 
Note that ``planning'' with a single agent as a capability model is essentially Bayesian inference.
The limitation with using capability models in planning, however, is that the synthesized plan is incomplete and abstract (see details later).
Nevertheless, capability models can be used to create approximate multi-agent plans. 
In our settings, we assume a centralized planner for now. 
Also, part of the agents are each modeled by a capability model as in Figure \ref{fig:cap-model}, 
and other agents are modeled with more complete models (e.g., action based models).
Agents with more complete models can be used to refine the plans when necessary.
This situation can naturally occur in human-robot teaming scenarios, 
in which human models need to be learnt and the robot models are given.
At the end of this paper, we present two problem formulations to demonstrate this.   

%

\section{Motivating Example}
\label{sec:example}

We start with a motivating example to demonstrate the motivations for the capability model. 
In this example, we have temporarily recruited human delivery agent, 
and the tasks involve delivering different parcels to their specified destinations. 
To deliver a parcel, an agent needs to first drive to a location that is close to the delivery location and then move the parcel to the location.
However, there are a few complexities here.
An agent may not be able to $carry\_parcel$, since the agent may not be strong enough. 
While the agent can use a trolley to $load\_parcel$, the agent may or may not remember to bring a trolley. 
The agent can visit a local rental office to $rent\_trolley$,
but there is a chance that the agent may have forgotten to bring money for the rental. 
The trolley increases the probability of the agent being able to move the parcel to the delivery location. 
The rented trolley needs to be returned by the agent after use at the end of the day. 
Figure \ref{fig:cap-model} presents the capability model for an individual human delivery agent. 
 
A capability model is a generalized $2$-TBN. 
Meanwhile, one important difference between capability model and $2$-TBN is,
while the nodes in the same time slice in a $2$-TBN are assumed to be synchronous,
fact nodes (or e-nodes) in capability models are not necessarily synchronous within the (partial) initial and final states of an operation.
This can also be seen from Figure \ref{fig:cap-model}, 
in which the edges between fact nodes (or e-nodes) can be labeled with actions. 
For example, given a capability specification,
suppose that $has\_money(AG)$ is {\em true} in the initial state, and $delivered(parcel)$ is {\em true} in the final state.
If a trolley is used in an operation for this capability, in which $has\_trolley(AG)$ becomes {\em true} only after applying the action $rent\_trolley$,
$has\_trolley(AG)$ and $has\_money(AG)$ are clearly not synchronous in the initial state of this operation.
Capability models allow us to encode the probabilities of transitions between two partial states.
 
To learn the agent models from experiences, the agents must be given a set of delivery tasks. 
However, the only information that the manager has access to may be the number of parcels that have been delivered. 
While this provides significant learning challenges for the parameters of MDP based and action based models,  
the useful information for the manager is already encoded: the probability that the agent can deliver a parcel.
Capability model can learn this information from the incomplete traces.
When more information is provided, it can be used to refine the capability models as with the other models. 
One of our goals is to provide a model that can learn useful information,
subject to various levels of incompleteness in plan execution traces. 

Now, suppose that the manager (i.e., a centralized planner) also has a set of robots (with action based models) to use that can bring goods into the delivery van. 
Observing that the presence of trolley increases the probability of delivery success, 
the manager can make a multi-agent plan in which the robot always ensures that there is a trolley in the van,
before the delivery agents start working.
This illustrates how capability models can be combined with other more complete models to refine the plans.

\section{Related Work}
\label{related}

In MDP based representation, each node represents a state, 
and each edge represents a transition between the two connected states after applying an action.
In action based representation, each action has a set of preconditions that must be satisfied in the current state, 
and a set of effects that specify changes to the current state when the action is applied. 
Assuming that the model structures are given (as with capability models) by the domain expert, 
the learning task involves only learning the model parameters.
The learning is performed given a set of the training examples, 
which are plan execution traces of the agent whose model is being learnt.

For MDP based representation, the agent is often modeled as a Partially Observable MDP \cite{Kaelbling:1996}) or POMDP,
with an underlying Hidden Markov Model (HMM).
The learning process needs to learn the transition probabilities between the states, 
as well as the probabilities of observations given the hidden states. 
Although researchers have investigated efficient learning methods \cite{Strehl07onlinelinear}, 
in general, learning in such models is still an intense computational task. 
Online learning methods for POMDP models have also been explored \cite{Shani05model-basedonline}.
One note is that capability model can be considered as a special factored-MDP,
which is often represented as a dynamic Bayesian network \cite{Sanner_relationaldynamic}.

In terms of learning for action based models, there are many previous works that discuss about learning the structures of actions, 
e.g., \cite{Zhuo:2013}.
As discussed previously, learning in our context is more about learning the parameters. 
In this aspect, there are not many works about parameter learning in action based models.

Meanwhile, both MDP based and action based models have been used in multi-agent planning. 
While Dec-POMDP \cite{Seuken:2007} and other related models (e.g., \cite{Gmytrasiewicz:2004}) have been popular, 
multi-agent planning for action based models has so far been concentrated on planning with deterministic models \cite{brafman2008}.
Planning with multiple agents using action based models in stochastic domains still remains to be investigated.

Capability model has connections to HTNs \cite{Erol_htnplanning}, since both provide flexible abstractions.
Finally, although there are other works that discuss about capability models, e.g., \cite{JenniferAAAI148476},
they are still based on action modeling.

\section{Capability Model}
\label{sec:cap-model}

For simplicity, we assume in this paper that the state of the world is specified as a set of boolean variables $X$,
which includes the agent state. 
More specifically, for all $X_i \in X$, $X_i$ has domain $D(X_i) = \{${\em true}, {\em false}$\}$. 
First, we formally define capability for agent $\phi$, in which we denote the set of variables that are relevant to $\phi$ as $X_\phi \subseteq X$. 

\begin{definition}[Capability]
Given an agent $\phi$, and four subsets of variables of $X_\phi$, $A$, $B$, $C$ and $D$, 
a capability, specified as $C \wedge \neg D \rightarrow A \wedge \neg B$, 
is an assertion about the existence of plans with initial state $I$ and goal state $G$ that satisfy the following:\\
$\forall{X_i \in C}, I[X_i] = $ {\em true}, \\$\forall{X_i \in D}, I[X_i] = $ {\em false}, \\
$\forall{X_i \in A}, G[X_i] = $ {\em true},\\ $\forall{X_i \in B}, G[X_i] = $ {\em false}.
\end{definition}
in which $S[X_i]$ represents the value of $X_i$ in state $S$. 

We call a plan that satisfies the above specification of a capability as an {\em operation}. 
Recall that the capability model captures the probabilities of the existence of an operation based on the specifications of capabilities,
given in the form of $C \wedge \neg D \rightarrow A \wedge \neg B$.
For brevity, we use $A$ to denote that $A = \{${\em true}$\}$ and $\neg A$ to denote that $A = \{${\em false}$\}$.
We construct the capability model of an agent as a Bayesian network from (potentially partial) causal relationships,
which are assumed to be provided by the domain writer. 
Due to this partiality, certain causal relationships among variables may not have been captured.
This also means that certain variables may not have been included in $X_\phi$ even when it should have. 
For example, whether an agent can drive a car to a new location is dependent on whether the agent can drive a manual car, 
even through the agent has a driver license.
However, the ability to drive a manual car may have been ignored by the domain expert when creating the capability model.

We model the capability model of each agent as an {\em augmented Bayesian network} \cite{neapolitan2004}.
We use augmented Bayesian network since it allows various types of prior beliefs to be specified. 
Furthermore, we make the {\em causal embedded faithfulness assumption} \cite{neapolitan2004}.
This assumption assumes that the probability distribution of the observed variables is encoded accurately 
(i.e., embedded faithfully \cite{neapolitan2004}) in a causal DAG that contains these variables and the hidden variables (i.e, variables that are not modeled).
While there are exceptions in which this assumption does not hold, 
their discussion is beyond the scope of this paper (see \cite{neapolitan2004} for details).
First, we introduce the definition of augmented Bayesian network. 

\begin{definition}
An augmented Bayesian network (ABN) $(G, F, \rho)$ is a Bayesian network with the following specifications:
\begin{itemize}
	\item A DAG $G = (V, E)$, where $V$ is a set of random variables, $V = \{V_1, V_2, ..., V_n\}$.
	\item $\forall V_i \in V$, an auxiliary parent variable $F_i \in F$ of $V_i$,
	 and a density function $\rho_i$ associated with $F_i$. 
	Each $F_i$ is a root and it is only connected to $V_i$. 
	\item $\forall V_i \in V$, for all values $pa_{i}$ of the parents $PA_i \subseteq V$ of $V_i$, and for all values $f_{i}$ of $F_i$, 
	a probability distribution $P(V_i | pa_{i}, f_{i})$. 
\end{itemize}
\label{def:abn}
\end{definition}

A capability model of an agent then is defined as follows:

\begin{definition}[Capability Model]
A capability model of an agent $\phi$, as a binomial ABN $(G, F, \rho)$, has the following specifications:
\begin{itemize}
	\item $V_\phi = X_\phi \cup \dot{X_\phi}$.
	\item $\forall X_i \in V$, the domain of $X_i$ is $D(X_i) = \{$true, false$\}$.
	\item $\forall X_i \in V$, $F_i = \{F_{i1}, F_{i2},...\}$, and each $F_{ij}$ is a root and has a density function $\rho_{ij}(f_{ij})$ ($0 \leq f_{ij} \leq 1$).
	\item $\forall X_i \in V$, $P(X_i = true | pa_{ij}, f_{i1},...f_{ij},...) = f_{ij}$.   
\end{itemize}
\label{def:cap}
\end{definition}
in which $j$ in $pa_{ij}$ indexes into the values of $PA_i$, while $j$ in $F_i$ indexes into the $j$th variable.  
$\dot{X_\phi}$ represents the set of {\em eventual} nodes or e-nodes that correspond (i.e., one-to-one mapping) to the nodes in $X_\phi$.
Recall that e-nodes represent the final value of a variable after an operation.

As discussed previously, the edges in the capability model of an agent is constructed based on (potentially partial) causal relationships. 
We assume that there are no causal feedback loops. 
Otherwise, loops can be broken randomly. 
Denote the set of edges that represent the given causal relationships between variables in $X_\phi$ as $Y_\phi$.
For each edge $X_i \rightarrow X_j \in Y_\phi$, we also add an edge $X_i \rightarrow \dot{X}_j$.
Furthermore, for each $X_i \in X_\phi$, we also add an edge $X_i \rightarrow \dot{X}_i$.
Denote the set of edges added as $\dot{Y}_\phi$. 
We then have $G_\phi = (V_\phi, E_\phi)$ in the capability model satisfy $E_\phi = Y_\phi \cup \dot{Y}_\phi$.
Figure \ref{fig:cap-model} provides a simple example of a capability model. 

\section{Model Learning}

The learning of the capability model (for agent $\phi$) is performed online through Bayesian learning.
The parameters (i.e., $f_{ij}$ for $F$) in the capability models of agents can be learnt from the previous plan execution traces.
Plan execution traces can be collected each time that an operation succeeds or fails. 
When the values of this set of parameters are assumed, the conditional distributions in Def. \ref{def:cap} are also known.

\begin{definition}[Plan Execution Trace]
A plan execution trace is a sequence of discontinuous state observations, $\mathcal{T} = \{S_1, S_i, ..., S_j, S_K\}$. 
\end{definition}

In the worse case, we only assume that $S_1$ and $S_K$ are given. 
Note that since the observations are discontinuous, the transition between contiguous state observations is not necessarily the result of a single action. 
When more than the initial and final states are given in the trace, it can be considered as a set of traces, 
broken in the form of $\{\{S_1, S_i\},..., \{S_j, S_K\}\}$.
Furthermore, when observations of the state are incomplete, 
the simplest solution is to generate all possible compatible plan traces with complete state specifications.  
We assume that the new plan execution traces are processed as above, and a training set $D$ is generated.  

A common way is to model $F_{ij}$ using a {\em beta} distribution (i.e., as its density function $\rho$).
Refer to \cite{neapolitan2004} for arguments for this distribution. 
Denote the parameters for the beta distribution of $F_{ij}$ as $a_{ij}$ and $b_{ij}$.

Suppose that the initial values or the current values for $a_{ij}$ and $b_{ij}$ are given.
The only remaining task is to update $a_{ij}$ and $b_{ij}$ from the traces.
Given the training set $D$,
we can now follow Bayesian inference to update the parameters of $F_{ij}$ as follows:

Initially or currently,
\begin{equation}
	\rho(f_{ij}) = beta(f_{ij}; a_{ij}, b_{ij})
\end{equation}
After observing new training examples $D$, we have:
\begin{equation}
	\rho(f_{ij}|D) = beta(f_{ij}; a_{ij} + s_{ij}, b_{ij} + t_{ij})
\end{equation}
in which $s_{ij}$ is the number in which $X_i \in V_\phi$ is equal to {\em true} with $PA_i$ taking the value of $pa_{ij}$,
and $t_{ij}$ is the number in which it equals {\em false} with $PA_i$ taking the value of $pa_{ij}$.

\section{More about Capability Models}

As discussed, the capability model is also a two level dynamic Bayesian network, 
except that what connects the two levels are operations rather than actions.  

\begin{theorem}
Planning with a single agent using capability model is NP-hard, but is polynomial when the capability model is a Polytree. 
\end{theorem}

The proof is straightforward since planning with a single agent using capability model is equivalent to Bayesian inference.
However, as we discussed previously, the synthesized plan is incomplete and abstract given the inherent incomplete and abstract nature of such models. 
In the single agent case, the output is only a probability measure with a single operation (i.e., the most abstract plan).
For multi-agent planning, this incompleteness and abstraction introduce a few complexities when using the model .

Suppose that we plan to use an operation with the capability specification of $C \wedge \neg D \rightarrow A \wedge \neg B$.
Through Bayesian inference, the capability model would return the corresponding probability of $P(\dot{A} \wedge \neg \dot{B} | C \wedge \neg D)$.\footnote{Note that $\dot{A}$, representing $\wedge_{a \in A}\dot{a} = $ {\em true}, is not equivalent to $\forall a \in A, a = $ {\em true} after an operation. Nevertheless, we assume that, e.g., $P(\dot{A} | C)$, is an approximation of the probability that an operation exists, which satisfies the capability specification of $C \rightarrow A$ (i.e., all variables in $A$ become {\em true} after this operation).}
We use the corresponding variables (i.e., e-nodes) here for $A$ and $B$, since they represent the same variables after the operation. 
This measure represents the probability of the existence of an operation for the specification of $C \wedge \neg D \rightarrow A \wedge \neg B$.
It means that this operation, unfortunately, may fail with certain probability as well. 

Furthermore, note that here, nothing about the values of other variables are specified in the (partial) initial and final states in the specification of this operation.
This means that the agent may be relying on other variables (that are not in $C$ or $D$) in the complete initial state to execute this operation 
(which may cause a failure for this operation), 
and that the agent can change the values of other variables (that are not in $A$ or $B$) in the complete final state after this operation.
To reduce the possible outcomes of an operation, we assume that the agent is rational, 
such that the agent would only change variables that are related to the achievement of $A$ and $\neg B$, 
i.e., $AC_\phi(A \cup B)$,
which represents the union set of ancestors of $A$ and $B$ in $X_\phi$ for agent $\phi$. 

In the following, we note a few properties that the capability model satisfies:  

{\em Monotonicity}: 
\begin{equation}
	P(\dot{A}, \neg \dot{B}|C, \neg D) \geq P(\dot{A'}, \neg \dot{B}|C, \neg D) (A \subseteq A')
\label{equ:mono1}
\end{equation} 
\begin{equation}
	P(\dot{A}, \neg \dot{B}|C, \neg D) \geq P(\dot{A}, \neg \dot{B'}|C, \neg D) (B \subseteq B')
\label{equ:mono2}
\end{equation} 

These imply that it is always more difficult to achieve the specified values for more variables.
Another aspect of monotonicity follows:
\begin{equation}
	P(\dot{A}, \neg \dot{B}|C, \neg D) \geq P(\dot{A}, \neg \dot{B}|C', \neg D) (C \supseteq C)
\label{equ:mono3}
\end{equation} 
\begin{equation}
	P(\dot{A}, \neg \dot{B}|C, \neg D) \geq P(\dot{A}, \neg \dot{B}|C, \neg D') (D \subseteq D')
\label{equ:mono4}
\end{equation} 

These imply that it is always easier to achieve the desired state with more {\em true}-value variables. 
This property holds with the assumption that no negated propositions appear in the preconditions of actions.
Note that this assumption does not reduce the generality of the domain, 
since negated propositions can be compiled away.

\section{MAP-MM}
\label{sec:map-mm}

In this section and next, we concentrate on how capability models can be used along with other more complete models to refine the synthesized plans. 
The settings are similar to our motivating example. 
In the first problem, which we call multi-agent planning with mixed models (MAP-MM). 
This formulation is useful especially for applications that involve both human and robot agents. 
While robot agents are programmed and hence have given models (we assume action based models here), 
the models of human agents must be learnt.
Hence, we use the agent capability model to model human agents and use the STRIPS action model to model robot agents.
We assume that the models for the human agents are already learnt in the following discussions. 

For robot agents, we assume the STRIPS action model $\langle \mathcal{R}, \mathcal{O} \rangle$, in which $\mathcal{R}$ is a set of predicates with typed variables, $\mathcal{O}$ is a set of STRIPS operators.
Each operator $o \in \mathcal{O}$ is associated with a set of preconditions $Pre(o) \subseteq \mathcal{R}$, add effects $Add(o) \subseteq \mathcal{R}$ and delete effects $Del(o) \subseteq \mathcal{R}$.

\begin{definition}
Given a set of robots $R = \{r\}$, a set of human agents $\Phi = \{\phi\}$, and a set of typed objects $O$, a multi-agent planning problem with mixed models is given by a tuple $\Pi = \langle X, \Phi, R, I, U, G \rangle$, where:
	\begin{itemize}
		\item $X$ is a finite set of propositions (i.e., corresponding to the set of variables specifying the world state) instantiated from predicates $\mathcal{R}$ and objects $O$; $I \subseteq X$ encodes the set of propositions that are true in the initial state; $U \subseteq X$ encodes the set of propositions that are unknown; other propositions (i.e., $X \setminus I \setminus U$) are assumed to be false; $G \subseteq X$ encodes the set of goal propositions.
		\item $A(r)$ is the set of actions that are instantiated from $\mathcal{O}$ and $O$, which $r \in R$ can perform.
		\item $\mathcal{B}_\phi$ is the capability model (see Figure \ref{fig:cap-model} for an example) of $\phi \in {\Phi}$, and $V_\phi \subseteq X$.
	\end{itemize}
\label{def:hat}
\end{definition}

The solution for a MAP-MM problem is a plan that maximizes the probability of success.  

\begin{theorem}
The MAP-MM problem is at least PSPACE-complete.
\end{theorem}

\begin{proof}
We only need to prove the result for one of the extreme cases: 
when there are only robot agents. 
The problem then essentially becomes a multi-agent planning problem,
which is more general than the classical planning problem, which is known to be PSPACE-complete.
\end{proof}

\section{Planning for MAP-MM}

We discuss in this section how to use an $A^*$ search process to perform planning for a MAP-MM problem . 
The current planning state $S$ is composed of three sets of propositions, denoted as $T(S)$, $N(S)$, and $U(S)$.
$T(S)$ is the set of propositions known to be {\em true} in the current state, $N(S)$ is the set of propositions known to be {\em false}, 
and $U(S)$ is the set of propositions that are {\em unknown}. 
Given the current state $S$, the (centralized) planner can expand it in the next step using the following options: 

\begin{itemize}
	\item Choose an action $a \in A(r) $ such that $Pre(a) \subseteq T(S)$.
	\item Choose a human operation on human agent $\phi$ with the specification $C \wedge \neg D \rightarrow A \wedge \neg B$, such that $C \subseteq T(S)$ and $D \subseteq N(S)$.
\end{itemize}

For an action $a$ from the robot agent, the planning state after applying $a$ becomes $S'$, 
such that $T(S') = T(S) \cup Add(a) \setminus Del(a)$, 
$N(S') = N(S) \cup Del(a) \setminus Add(a)$,
and $U(S') = U(S) \setminus Add(a) \setminus Del(a)$. 
Note that the number of unknown variables is reduced after applying this action.

For a human operation on agent $\phi$, 
the planning state after applying the operation becomes $T(S') = T(S) \cup A \setminus B \setminus AC_\phi(A \cup B)$,
$N(S') = N(S) \cup B \setminus A \setminus AC_\phi(A \cup B)$, 
and $U(S') = U(S) \cup AC_\phi(A \cup B) \setminus A \setminus B$.
The number of unknown variables increases.

To compute the $f$ (i.e., cost) values for the $A^*$ search, 
we take the negative logarithms of the associated probabilities of reachable planning states.
Each planning state that is reachable from the initial state is associated with a probability of success, 
which is the product of the probabilities of success of the operations used to reach the current state from the initial state, 
given the partial plan constructed.
Hence, the $g$ value can be computed straightforwardly. 
To compute an admissible $h$ value to each planning state, we denote the set of goal propositions that are currently
{\em false} or {\em unknown}, and that can only be added by a human agent as $G_\Phi$. 
Denote the set of human agents that can add $p \in G_{\Phi}$ by $\Phi_p$
We first compute the human agent with the minimum cost to achieve $p \in G_\Phi$ based on the following formula:
\begin{equation}
	\phi_p = \argmax_{\phi \in \Phi_p}-\log{P(p | X_{\phi} \setminus p)}
\label{equ:selectmax}
\end{equation}

We then compute the $h$ value as:
\begin{equation}
	h(S) = \max_{p \in G_{\Phi}} -\log{P(p | X_{\phi_p} \setminus p)}
\label{equ:heuristic}
\end{equation}

\begin{lemma}
The heuristic given in Eq. \eqref{equ:heuristic} is admissible. 
\label{lem:admissible}
\end{lemma}

\begin{proof}
We need to prove that $h(S)$ is not an over-estimate of the cost for any state $S$, given $G$. 
If $G_\Phi$ is empty, $h(S) = 0$, and it must not be an over-estimate given non-negative cost measures. 
Otherwise, the propositions in $G_\Phi$ must be made {\em true} by a human agent. 
The minimum cost that has to be incurred is no less than the maximum cost for making an individual proposition {\em true}, 
given that all the other propositions are {\em true}. 
This result is based on the {\em monotonicity} property. 
Hence, the conclusion holds. 
\end{proof}

\begin{lemma}
The heuristic given in Eq. \eqref{equ:heuristic} is consistent. 
\label{lem:consistent}
\end{lemma}

\begin{proof}
We need to prove that $h(S) - h(S') \leq dist(S, S')$ (interpreted as from $S$ to $S'$).
Note the asymmetry of the distance measure here. 
From Eq. \eqref{equ:heuristic}, suppose that $h(S) = -\log{P(p | X_{\phi_p} \setminus p)}$ and that $h(S') = -\log{P(p' | X_{{\phi'}_{p'}} \setminus p')}$.
If $p = p'$, clearly, we have $h(S) = h(S')$, and given non-negative distance measures, the conclusion holds.
Otherwise, first, we know that $S$ does not contain $p$ and $S'$ does not contain $p'$.
Furthermore, we cannot have that both $S$ and $S'$ contain neither $p$ or $p'$, given Eq. \eqref{equ:selectmax},
since otherwise we would have $h(S) = h(S')$, which implies that $p = p'$.
Without loss of generality, assume that $h(S) \geq h(S')$.
As a result, we must have that $S'$ contains $p$, since otherwise we would have chosen $p$ in $h(S')$.
Hence, $dist(S, S')$ must be at least the cost for making $p$ {\em true} in $S$.
Hence, based on the {\em monotonicity} property, the conclusion holds. 
\end{proof}

Lem. \ref{lem:admissible} and \ref{lem:consistent} ensure that the $A^*$ search is {\em optimally efficient}, 
meaning that the search expands only the necessary nodes (i.e., planning state) in order to find the optimal solution.  
Nevertheless, given the complexity, approximation solutions must be provided to scale to large problem instances.   

\section{MAP-MMI}

There are many extensions to the MAP-MM problem. 
One immediate extension is to consider interaction cost between the human and robot agents, which we call MAP-MMI (for interaction-aware MAP-MM). 
While MAP-MM is useful for a plan that is most likely to succeed,
MAP-MM formulation is useful when agents must coordinate to execute a conditional plan during the {\em execution} phase,
and communication is limited. 
In particular, the human agents need to be informed when to execute an operation. 

In such scenarios, the plan made is similar to a conditional plan.
More specifically, the plan should specify the planning branches depending on whether an operation succeeds or fails.
During plan execution, an actual request to perform the operation would be made. 
A failure would be reported when that operation fails to be executed in the current situation.

In MAP-MMI, for each request of a human operation (which is planned during the planning phase but incurred in the execution phase), there is a cost associated. 
For simplicity, we assume that the cost is constant per request. 
Since the communication is limited, only a fixed number of requests may be made.

\begin{definition}
A MAP-MMI problem is a MAP-MM problem with a communication threshold $\mathcal{C}$. 
The solution is a conditional plan with the maximum probability of success while keeping the number of requests (i.e., interaction cost) below or equal to $\mathcal{C}$.
\end{definition}

\begin{corollary}
MAP-MMI is at least PSPACE-complete.
\end{corollary}

\begin{proof}
Since MAP-MMI can be considered as a special case of MAP-MM with a constraint on the interaction cost, the conclusion is straightforward. 
More specifically, we can reduce a MAP-MM problem to a MAP-MMI problem. 
Suppose that we have an efficient solution for MAP-MMI.
We can derive an efficient solution for any MAP-MM problem by simply ending operation failures in undefined states, 
while assigning a large value to $\mathcal{C}$. 
It is not difficult to see that the solution to this MAP-MMI problem is also the solution for the original MAP-MM problem.
\end{proof}

Planning for MAP-MMI can be performed in a similar way as in MAP-MM.
The major differences lie in how a planning state $S$ is represented and how it can be expanded. 
In particular, a state $S$ in MAP-MMI contains more than one substate to be expanded, since we are creating a conditional plan.
Denote $s \in S$ as one substate (which is similar to a planning state in MAP-MM), and each such substate is associated with a probability of success.  
We can perform the similar expansions to each $s$ in $S$ as in MAP-MM, except that the resulting states are different, 
and that the number of requests from each state $S$ is upper-bounded by $\mathcal{C}$.

When an action is chosen from $a \in A(r)$, $s$ is expanded similarly as in MAP-MM.
When a human operation on agent $\phi$ is used in the form of $C \wedge \neg D \rightarrow A \wedge \neg B$, 
$s$ is expanded into two substates.
One substate represents that the operation succeeds, and the other one representing a failure. 
The probability of success with $s$ is also split onto these two substates based on the probability
of success of the chosen operation computed from $\mathcal{B}_\phi$.
While the substate for a successful operation is updated according to our previous discussion for MAP-MM, 
the substate when the operation fails is updated in a different manner, 
since $A$ may not have been achieved.
In particular, the new substate $s'$ is updated such that $T(s') = T(s) \setminus AC_\phi(A \cup B) \setminus A \setminus B$,
$N(s') = N(s) \setminus AC_\phi(A \cup B) \setminus A \setminus B$, 
and $U(s') = U(s) \cup A \cup B \cup AC_\phi(A \cup B)$.

The heuristic cost for $h$ can be used for each substate as in Eq. \eqref{equ:heuristic}. 
To combine the heuristic costs for different substates, we need to first convert the costs to a probability measures,
and then sum these converted measures to obtain the probability of success for $S$. 
This probability can then be converted to a cost measure to obtain the $h$ value for $S$.
The admissibility and consistency properties then follow immediately from Lem. \ref{lem:admissible} and \ref{lem:consistent}.

\section{Conclusions}
\label{sec:con}

In this paper, we introduce a new representation to model agents based on capability models, which has several unique advantages. 
The underlying structures of these models are generalized $2$-TBNs, 
and we assume that the (potentially partial) causal relationships between the nodes in the networks are provided by the domain expert. 
First, we show that learning capability models can be performed efficiently online via Bayesian learning, 
and the learning process is robust to high degrees of incompleteness in plan execution traces (e.g., with only start and end states). 
Furthermore, when used in multi-agent planning (with each agent modeled by a capability model), our models provide flexible abstraction of actions, 
which can potentially improve planning performance.
The limitation, however, is that the synthesized plan is incomplete and abstract. 
At each step, instead of specifying which action to execute as with other models, 
the plan only specifies a partial state to be achieved. 
We provide details on how these models are constructed and how the parameters are learnt and updated online. 
Moreover, we introduce two multi-agent planning problems that use these models, and provide heuristic planning methods for them.
These problems also illustrate how capability models can be combined with other more complete models to refine the plans. 

\bibliography{paper.bib} 
\bibliographystyle{aaai}

\end{document}